\documentclass[]{article}
\usepackage{proceed2e}

\usepackage[usenames]{color}
\usepackage{enumerate}
\usepackage{graphicx}
\usepackage{amsmath}
\usepackage{amsthm}
\usepackage{algorithm}
\usepackage{algorithmic}
\usepackage{xspace}
\usepackage{subfigure}
\usepackage{apacite}
\usepackage{verbatim}

\newcommand{\TODO}[1]{\textbf{\textsc{\textcolor{red}{(TODO: #1)}}}}
\newcommand{\commentout}[1]{}

\newcommand{\secref}[1]{Section~\ref{#1}}
\newcommand{\figref}[1]{Fig.~\ref{#1}}

\newcommand{\var}{x}
\newcommand{\Var}{X}
\newcommand{\varvec}{\mathbf{\var}}

\newcommand{\potfun}{\Theta}
\newcommand{\potval}{\theta}
\newcommand{\graph}{\mathcal{G}}
\newcommand{\nodes}{\mathcal{V}}
\newcommand{\edges}{\mathcal{E}}
\newcommand{\vx}{\mathbf{x}}

\newcommand{\edgesgc}{\mathcal{E}^{GC}}
\newcommand{\nodesgc}{\mathcal{V}^{GC}}
\newcommand{\graphgc}{\mathcal{GC}}

\newcommand{\true}[1]{\mathbf{1}[#1]}

\newcommand{\graphfg}{\mathcal{FG}}
\newcommand{\edgesfg}{\mathcal{E}^{FG}}

\newcommand{\gcvar}{v}

\newcommand{\msg}[2]{m_{{#1}\rightarrow{#2}}}
\newcommand{\neighbors}{\mathcal{N}}
\newcommand{\belief}{b}

\newcommand{\ouralg}{APMP}
\newcommand{\inbalance}{edge-calibrated}
\newcommand{\messages}{\mathbf{m}}
\newcommand{\augpath}{\mathcal{T}}
\newcommand{\path}{\augpath}
\newcommand{\paths}{\mathcal{P}}

\newcommand{\enU}{U}
\newcommand{\enR}{R}
\newcommand{\supu}{(U)}
\newcommand{\supr}{(R)}

\newcommand{\potfunq}{\potfun^{(U)}}
\newcommand{\potfunu}{\potfun^{(U)}}

\newcommand{\potfunr}{\potfun^{(R)}}

\newcommand{\potvalr}{\potval^{(R)}}
\newcommand{\potvalu}{\potval^{(U)}}

\newtheorem{lem}{Lemma}
\newtheorem{thm}{Theorem}
\newtheorem{cor}{Corollary}
\newtheorem{deff}{Definition}

\title{Interpreting Graph Cuts as a Max-Product Algorithm} 

\author{ {\bf Daniel Tarlow, Inmar E. Givoni, Richard S. Zemel, Brendan J. Frey} \\  
University of Toronto\\ 
Toronto, ON  M5S 3G4  \\ 
\url{ \{dtarlow@cs, inmar@psi, zemel@cs, frey@psi\}.toronto.edu}
} 

\begin{document} 
 
\maketitle

%

%

\begin{abstract}
The maximum a posteriori (MAP) configuration of binary variable models with submodular
graph-structured energy functions can be found efficiently and exactly by graph cuts. 
Max-product belief propagation (MP) has been shown to be suboptimal on this class 
of energy functions by a canonical counterexample where MP converges to a suboptimal 
fixed point \shortcite{Kulesza08}.

In this work, we show that under a particular scheduling and damping scheme, MP 
is equivalent to graph cuts, and thus optimal.  We explain the apparent contradiction 
by  showing that with proper scheduling and damping, MP 
always converges to an optimal fixed point.  Thus, 
the canonical counterexample only shows 
the suboptimality of MP with a particular suboptimal choice of schedule and damping.  With 
proper choices, MP is optimal.
\end{abstract}

\section{Introduction}
Maximum a posteriori (MAP) inference in probabilistic graphical models is a fundamental machine learning
task with applications to fields such as computer vision and computational biology.
There are various algorithms designed to solve MAP problems, each
providing different problem-dependent theoretical guarantees and empirical performance.
  It is often difficult to choose which
 algorithm to use in a particular application.  In some cases, however, there is
 a ``gold-standard'' algorithm that clearly outperforms competing algorithms, such as
 the case of graph cuts for binary submodular problems.\footnote{From here on, we drop
 ``graph-structured'' and refer to the energy functions just as binary submodular.  Unless
 explicitly specified otherwise, though, we
 always assume that energies are defined on a simple graph.} 
A popular and more general, but also occasionally erratic, algorithm is
max-product belief propagation (MP).

Our aim in this work is to 
establish the precise relationship between MP and graph cuts, namely that graph cuts is a special case of MP.
To do so, we map analogous aspects of the algorithms to each other:\
message scheduling in MP to selecting augmenting paths in graph cuts; passing messages on a chain to pushing flow through an augmenting path; message damping to limiting flow to be the bottleneck capacity of an augmenting path;
and letting messages reinforce themselves on a loopy graph to the graph cuts 
connected components decoding scheme.

This equivalence implies strong statements regarding the optimality of MP
on binary submodular energies defined on graphs with arbitrary topology, which
 may appear to contradict much of what is known about MP---all empirical 
results showing MP to be suboptimal on binary submodular problems, and the theoretical results 
of \shortciteA{Kulesza08, WaiJor08} which show analytically that MP converges to 
the wrong solution.  
We analyze this issue in depth and show there is no contradiction, but implicit in the previous analysis
and experiments is a suboptimal choice of scheduling and damping, leading the algorithms to converge
to bad fixed points.
Our results give a more complete characterization of these issues, 
showing (a) there always exists an optimal fixed point for binary submodular energy functions,
 and (b) with proper scheduling and damping 
MP can always be made to converge to an optimal fixed point.


The \emph{existence} of the optimal MP fixed point can alternatively be derived as a 
consequence of the analysis of the zero temperature limit of convexified 
sum-product in \shortciteA{Weiss07} along with the well-known fact 
that the standard linear program relaxation is tight for binary submodular 
energies. Our proof of the existence of the fixed point, then,  is an alternative, 
more direct proof.  However, we believe our \emph{construction} of the fixed point to 
be novel and significant, particularly due to the fact that the construction 
comes from simply running ordinary max-product within the standard algorithmic degrees 
of freedom, namely damping and scheduling.

Our analysis is significant for many reasons.  Two of the most important are as follows.
First, it shows that previous constructions of MP fixed points for binary submodular
energy functions critically depend on the particular schedule, damping, and initialization.  
Though there exist suboptimal fixed points, there also
always exist optimal fixed points, and with proper care, the bad fixed points can always
be avoided.  Second, it simplifies the space of MAP inference algorithms, 
making explicit the connection between two popular and seemingly 
distinct algorithms.  The mapping improves our understanding of message scheduling and
gives insight into how graph cut-like algorithms might be developed for more general
settings.

\section{Background and Notation}\label{sec:bg_notation}

We are interested in finding maximizing assignments of distributions $P(\varvec) \propto e^{-E(\varvec)}$ where $\varvec = \{\var_1, \ldots, \var_M\} \in \{0, 1\}^M$.
We can equivalently seek to minimize the energy  $E$, and for the sake of exposition we choose to present the analysis in terms of energies\footnote{This makes ``max-product'' a bit of a misnomer, since in reality, we will be analyzing min-sum belief propagation.  The two are equivalent, however, so we will use ``max-product'' (MP) throughout, and it should be clear from context when we mean ``min-sum''.}.

{\bf Binary Submodular Energies: }
We restrict our attention to submodular energy functions over binary variables.
\emph{Graph-structured} energy functions are defined on a simple graph, $\graph = (\nodes, \edges)$, where each node is associated with a variable $\var$. Potential functions $\potfun_{i}$ and $\potfun_{ij}$ map configurations of individual variables and pairs of variables whose corresponding nodes share an edge, respectively, to real values.
We write this energy function as
\begin{eqnarray}
E(\vx;\potfun) & = &\sum_{i \in \nodes} \potfun_{i}(x_i) + \sum_{ij \in \edges} \potfun_{ij}(x_i, x_j) \hbox{.} \label{eq:main_energy}
\end{eqnarray}
%
%
$E$ is said to be \emph{submodular} if and only if for all $ij \in \edges $, 
$\potfun_{ij}(0,0) + \potfun_{ij}(1,1) \le \potfun_{ij}(0,1) + \potfun_{ij}(1,0)$. 
We use the shorthand notation $[\potval_i^0, \potval_i^1]=[\potfun_i(0), \potfun_i(1)]$.

When $E$ is submodular, it is always possible to represent all pairwise potentials in the canonical form
\begin{eqnarray*}
\left[ \begin{array}{cc}
					\potfun_{ij}(0, 0) & \potfun_{ij}(0, 1) \\
					\potfun_{ij}(1, 0) & \potfun_{ij}(1, 1)
				  \end{array} \right]
				   & = &  \left[ \begin{array}{cc}
					0 & \potval_{ij}^{01}  \\
					\potval_{ij}^{10} & 0
				  \end{array} \right]
\end{eqnarray*}
with $\potval_{ij}^{01}, \potval_{ij}^{10} \ge 0$ without changing the energy of any assignment.
We assume that energies are expressed in this form throughout.\footnote{See \shortciteA{Kolmogorov02} for a more thorough discussion of representational matters.}  In our notation, $\potval_{ij}^{01}$ and $\potval_{ji}^{10}$ refer to the same quantity.

\subsection{Graph Cuts}

Graph cuts is a well-known algorithm for minimizing graph-structured binary submodular energy
functions, which is known to converge to the optimal solution in low-order polynomial time by
transformation into a maximum network flow problem.
\commentout{
\shortcite{Grieg89}, were the first to apply graph cuts in computer vision, showing the reduction from the binary image denoising problem to the maximum network flow problem in an appropriately constructed graph. This initiated the use of known exact polynomial time combinatorial algorithms for finding MAP assignments.  Since then, graph cuts have been popular in machine learning and computer vision. (See \shortcite{Boykov04,Boykov06} and references therein.)

\textbf{Graph Construction:}
}
The energy function is converted into 
a weighted directed graph $\graphgc = (\nodesgc, \edgesgc, C)$,
where $C$ is an edge function that maps each directed edge $(i,j) \in \edgesgc$ to a non-negative real number
representing the initial capacity of the edge.
One \emph{non-terminal} node $\gcvar_i \in \nodesgc$ is constructed for each
variable $\var_i \in \nodes$, and two \emph{terminal} nodes, a source $s$, and a sink $t$, are added to $\nodesgc$.
Edges in $\edges$ are mapped to two edges in $\edgesgc$, one per direction.  The initial capacity
of the directed edge $(i,j) \in \edgesgc$ is set to $\potval_{ij}^{01}$, and
the initial capacity of the directed edge $(j, i) \in \edgesgc$ is set to $\potval_{ij}^{10}$.
In addition, directed edges are created from the source node to every non-terminal node, and from every non-terminal node to the sink node. The initial capacity of the terminal edge from $s$ to $\gcvar_i$ is set to be $\potval_{i}^1$, and the initial capacity of the terminal edge from $\gcvar_i$ to $t$ is set to be $\potval_{i}^0$.  
We assume that the energy function has been normalized so that one of the initial terminal edge capacities is 0 for every non-terminal node.

\commentout{
\textbf{Max-Flow Algorithms:}
After constructing the appropriate graph, the core of graph cuts is a maximum network flow problem, which is a well-studied
combinatorial optimization problem with history dating back to 1956 \shortcite{Ford56}.
The earliest algorithm for solving max-flow is due to \shortcite{Ford56}, which works by repeatedly finding any flow-bearing path from the {source} node to the {sink} node, pushing as much flow as the current path will allow.
For this reason, the algorithm is referred to as an \emph{augmenting paths} algorithm.
}

\textbf{Residual Graph:}
Throughout the course of an augmenting paths-based max-flow algorithm, \emph{residual capacities} 
(or equivalently hereafter, \emph{capacities}) are maintained
for each directed edge.  The residual capacity is the amount of flow that can
be pushed through an edge either by using unused capacity or by reversing flow
that has been pushed in the opposite direction.  Given a flow of $f_{ij}$ from $\gcvar_i$ to $\gcvar_j$ via
edge $(i, j)$ and a flow of $f_{ji}$ from $\gcvar_j$ to $\gcvar_i$ via
edge $(j,i)$, the residual capacity is  $r_{ij} = \potval_{ij}^{10} - f_{ij} + f_{ji}$.
An \emph{augmenting path} is a path from $s$ to $t$ through the residual graph
that has positive capacity.  We call the minimum residual capacity of any edge along an augmenting
path the \emph{bottleneck capacity} for the augmenting path. 


{\bf Two Phases of Graph Cuts}:
Augmenting path algorithms for graph cuts proceed in two phases.  
In Phase 1, flow is pushed through augmenting paths until
all source-connected nodes (i.e., those with an edge from source to node with positive capacity)
 are separated from all sink-connected nodes (i.e., those with an edge to the sink with positive capacity).
   In Phase 2, to determine
assignments, a
connected components algorithm is run to find all nodes that are reachable from the source
and sink, respectively.

{\bf Phase 1 -- Reparametrization}:
\commentout{
\begin{figure}[tb]
\centering
\subfigure[]{\label{fig:gc_rep1}\includegraphics[width=.2\columnwidth]{gc_reparametrization.pdf}}
\hspace{35mm}
\subfigure[]{\label{fig:fg_rep1}\includegraphics[width=.2\columnwidth]{fg_reparametrization1.pdf}}\\
\subfigure[]{\label{fig:gc_rep2}\includegraphics[width=.2\columnwidth]{gc_reparametrization2.pdf}}
\hspace{35mm}
\subfigure[]{\label{fig:fg_rep2}\includegraphics[width=.2\columnwidth]{fg_reparametrization2.pdf}}
\caption{
Pushing flow through (reparameterizing) an augmenting path.  Here, $f = 1$.
\subref{fig:gc_rep1} and \subref{fig:gc_rep2} show a representation of $\graphgc$, while \subref{fig:fg_rep1} and \subref{fig:fg_rep2} show unary
and pairwise potentials in a factor graph representation of each set
of edge capacities.
\subref{fig:gc_rep1},\subref{fig:fg_rep1} Before, there is a capacity of 1 on each edge along the path, and no capacity
in the backward direction.  \subref{fig:gc_rep2},\subref{fig:fg_rep2} After, capacity has been used up
in all forward direction edges, and the backward direction capacity has
been incremented by 1, which effectively lets us undo the decision to push
flow across the edge $(1, 2)$ if we later find a path that could use $(2, 1)$.
}
\label{fig:gc_reparametrization}
\end{figure}
}
The first phase 
can be
viewed as reparameterizing the energy function, moving mass from
unary and pairwise potentials to other pairwise potentials and from
unary potentials to a constant potential~\shortcite{Kohli07b}. 
The constant potential is a lower bound on the optimum. 

We begin by rewriting \eqref{eq:main_energy} as
\begin{align}
&E(\vx;\potfun)  =    \sum_{i \in \nodes} \theta^0_{i} (1 - x_i) + \sum_{i \in \nodes} \theta^1_{i} x_i 
+  \sum_{ij \in \edges} \theta^{01}_{ij} (1 - x_i) x_j \nonumber \\
 &   \;\;\; + \sum_{ij \in \edges} \theta^{10}_{ij} x_i (1 - x_j ) + \potval_{const} \label{eq:gc_energy_split}
\hbox{,}
\end{align}
%
where we added a constant term $\potval_{const}$, initially set to 0, to $E(\varvec;\potfun)$ without
changing the energy. 

A reparametrization
is a change in potentials from $\potfun$ to $\tilde \potfun$ such
that $E(\varvec; \potfun) = E(\varvec; \tilde \potfun)$
for all assignments $\varvec$.
%
Pushing flow corresponds to factoring out a constant, $f$, from some subset
of terms and applying the following algebraic identity to terms from \eqref{eq:gc_energy_split}:
\begin{align*}
& f \cdot \left[ x_1 + (1 - x_1) x_2  + \ldots + (1 - x_{N-1}) x_N + (1 - x_N) \right] \\
 =&\; f \cdot \left[ x_1 (1 - x_2) + \ldots + x_{N-1} (1  - x_N) + 1 \right] \hbox{.}
\end{align*}
By ensuring that $f$ is positive (choosing paths that can sustain flow),
the constant potential can be made to grow at each iteration.
When no paths exist with
nonzero $f$, $\theta_{const}$ is the optimal energy value \shortcite{Ford56}.

In terms of the individual coefficients, pushing flow through a path
corresponds to reparameterizing entries of the potentials 
 on an augmenting path:
\begin{eqnarray}
\theta^1_1 & := & \theta^1_1 - f  \label{eq:gc_reparametrization_start}\\
\theta^0_N & := & \theta^0_N - f \\
\theta^{01}_{ij} & := & \theta^{01}_{ij} - f \qquad \hbox{ for all $ij$ on path} \\
\theta^{10}_{ij} & := & \theta^{10}_{ij} + f \qquad \hbox{ for all $ij$ on path} \\
\theta_{const} & := & \theta_{const} + f \hbox{.} \label{eq:gc_reparametrization_end}
\end{eqnarray}

{\bf Phase 2 -- Connected Components}:
After no more paths can be found, most nodes will not be directly connected
to the source or the sink by an edge that has positive capacity in the residual graph.
In order to determine assignments, information must be propagated
from nodes that are directly connected to a terminal via positive capacity edges via non-terminal
nodes.  A connected components procedure is run, and any node that is (possibly indirectly)
connected to the sink is assigned label 0, and any node that is (possibly indirectly)
connected to the source is given label 1.  Nodes that are not connected to either terminal can
be given an arbitrary label without changing the energy of the configuration, so long as within a 
connected component the labels are consistent.
In practice, terminal-disconnected nodes are typically given label 0.

\subsection{Strict Max-Product Belief Propagation}\label{sec:mpbp}

Strict max-product belief propagation (Strict MP) is an iterative, local, message passing algorithm that can be used to find the MAP configuration of a distribution specified by a tree-structured graphical model.  The algorithm can equally be applied to loopy graphs. Employing the energy function notation, the algorithm is usually referred to as  min-sum. Using the factor-graph representation \shortcite{Kschischang01}, the iterative updates on simple graph-structured energies involves sending messages from factors to variables
\begin{align}
\msg{\potfun_{i}}{\var_i}(\var_i) & =
		\potfun_{i}(\var_i) \label{eq:strict_mp_singletons}\\
	\msg{\potfun_{ij}}{\var_j}(\var_j) &= \! \min_{\var_i}\!
	\left[
		\potfun_{ij}(\var_i, \var_j) +  \msg{\var_{i}}{\potfun_{ij}}(\var_{i})
	\right] \label{eq:strict_mp}
\end{align}
%
and from variables to factors,
$\msg{\var_i}{\potfun_{ij}}(\var_i) = \sum_{i' \in \neighbors(i) \backslash \{j\}} \msg{\potfun_{i'i}}{\var_i}(\var_{i})$,
%
%
%
%
where $\neighbors(i)$ is the set of neighbor variables of
$i$ in $\graph$.
In Strict MP, we require that all messages are updated in parallel in each iteration.
Assignments are typically decoded from beliefs as
$\hat \var_i = \arg \min_{\var_i} \belief_i(\var_i) \hbox{,}$
where
$\belief_i(\var_i) = \potfun_i(\var_i) + \sum_{j \in \neighbors(i)} \msg{\potfun_{ji}}{\var_i}(\var_{i})$.
Pairwise beliefs are defined as
$\belief_{ij}(\var_i,\var_j) = \potfun_{ij}(\var_i,\var_j) + \msg{\var_i}{\potfun_{ij}}(\var_{i})+ \msg{\var_j}{\potfun_{ij}}(\var_{j}) \hbox{.}$%
\footnote{Note that we only need message values to be correct up to a constant,
so 
it is
 common practice to normalize messages and beliefs so that
the minimum entry in a message or belief vector is 0.}

\commentout{
The graph cut graph, $\graphgc$, and the factor graph representation, $\graphfg$,
of the energy function are defined over similar variables and potentials,
so it is not surprising that there is a direct mapping between quantities
in the factor graph representation of the energy function and the graph cut representation
of the energy function.
}

\subsection{Max-Product Belief Propagation}\label{sec:mpbp_practical}
In practice, Strict MP does not converge well, so a combination of damping
and asynchronous message passing schemes is typically used.  Thus,
MP is actually a family of algorithms.  
We formally define the family as follows:
\begin{deff}[Max-Product Belief Propagation]
MP is a message passing algorithm that computes messages as in \eqref{eq:strict_mp}.
Messages may be initialized arbitrarily, scheduled in any (possibly dynamic) ordering,
and damped in any (possibly dynamic) manner, so long as the fixed points of
the algorithm are the same as the fixed points of Strict MP.
\end{deff}
We believe this definition to be broad enough to contain most algorithms that
are considered to be max-product, yet restrictive enough to exclude e.g., fundamentally
 different linear program-based algorithms like tree-reweighted max-product.  

%
There has been much work on scheduling messages, including a recent string of 
work on dynamic asynchronous scheduling \shortcite{Elidan06, Sutton07}, which shows that adaptive
schedules can lead to improved convergence. 
An equally important practical concern 
is message damping.  \shortciteA{Dueck10}, for example, discusses the importance
of damping in detail with respect to using MP for exemplar-based clustering (affinity propagation).
Our definition of MP includes these variants.


\subsection{Augmenting Path = Chain Subgraph}\label{sec:map_aug_path_to_chain}
Our scheduling makes use of dynamically chosen chains, which are
analogous to augmenting paths.
Formally, an augmenting path is a sequence of nodes
\begin{equation}
\mathcal{T} = (s, \gcvar_{\mathcal{T}_1}, \gcvar_{\mathcal{T}_2}, \ldots ,  \gcvar_{\mathcal{T}_{n-1}}, \gcvar_{\mathcal{T}_n},  t)
\label{eq:aug_path}
\end{equation}
where a nonzero amount of flow can be pushed through the path.  It will be useful to refer to
 $\edgesgc(\augpath)$ as the set of edges encountered along $\augpath$.

\commentout{
In the factor graph representation, we can define the variables  $\varvec_\mathcal{T} \subseteq \varvec$ corresponding to non-terminal nodes in $\augpath$.
In addition, we can define potentials corresponding to the edges $\edgesgc(\augpath)$ and entries of these
potentials that will be useful in the sequel.
Formally,
\begin{align}
\varvec_\mathcal{T}  = & \{x_{\mathcal{T}_1},\ldots,x_{\mathcal{T}_n}\} \nonumber \\
\potfun_\mathcal{T}  = & \potfun_{\mathcal{T}_1} \cup \{\potfun_{\mathcal{T}_i,\mathcal{T}_{i+1}}\}_{i=1}^{n-1} \cup \potfun_{\mathcal{T}_n} \nonumber \\
\potval_\mathcal{T}  = & \potval_{\mathcal{T}_1}^1 \cup \{\potval_{\mathcal{T}_i,\mathcal{T}_{i+1}}^{01}\}_{i=1}^{n-1} \cup \potval_{\mathcal{T}_n}^0 \hbox{.} \label{eq:augpath_quants}
\end{align}
}

Let $\varvec_\mathcal{T} \subseteq \varvec$ be the variables corresponding to non-terminal nodes in $\augpath$. The potentials corresponding to the edges $\edgesgc(\augpath)$ and the entries of these
potentials are denoted by
$\potfun_\mathcal{T}$ and  a subset of potential values $\potval_\mathcal{T}$.
Formally,
\begin{align}
\varvec_\mathcal{T}  = & \{x_{\mathcal{T}_1},\ldots,x_{\mathcal{T}_n}\} \nonumber \\
\potfun_\mathcal{T}  = & \potfun_{\mathcal{T}_1} \cup \{\potfun_{\mathcal{T}_i,\mathcal{T}_{i+1}}\}_{i=1}^{n-1} \cup \potfun_{\mathcal{T}_n} \nonumber \\
\potval_\mathcal{T}  = & \potval_{\mathcal{T}_1}^1 \cup \{\potval_{\mathcal{T}_i,\mathcal{T}_{i+1}}^{01}\}_{i=1}^{n-1} \cup \potval_{\mathcal{T}_n}^0 \hbox{.} \label{eq:augpath_quants}
\end{align}
Note that there are only two unary potentials on a chain corresponding to an augmenting path, which 
correspond to terminal edges in $\edgesgc(\augpath)$.  
It will be useful 
to map edges in $\edgesgc(\augpath)$ to edges in the equivalent factor graph representation.  We
use $\edgesfg(\augpath)$ to denote all edges in $\graph$ between potentials in $\potfun_\augpath$ and
variables in $\varvec_\mathcal{T}$.  

As an example, an augmenting path $\mathcal{T}= (s, \gcvar_i, \gcvar_j , \gcvar_k , t)$ in the graph cut formulation would be mapped to $\varvec_\mathcal{T}=\{\var_i,\var_j,\var_k\}$, $\potfun_\mathcal{T} =\{ \potfun_i,\potfun_{ij},\potfun_{jk},\potfun_k\},$ and $\potval_\mathcal{T} =\{ \potval_i^1,\potval_{ij}^{01},\potval_{jk}^{01},\potval_k^0\}$.

\section{Augmenting Paths Max-Product} \label{sec:mp_apmp}



In this section, we present Augmenting Paths Max-Product (\ouralg), a particular
scheduling and damping scheme for MP, that---like graph cuts---has
two phases.  At each iteration of the first phase, the scheduler returns a chain on which
to pass messages. Holding all other messages fixed,
messages are passed forward and backward on the chain, with
standard message normalization applied,
to complete the iteration.  Adaptive
message damping applied to messages leaving unary
factors (described below) ensures that messages
propagate across the chain in a particularly structured way.
Messages leaving pairwise
factors and messages from variables to factors are not damped.
Phase 1 terminates when the scheduler
indicates there are no more messages to send, then in Phase 2, Strict MP is run until
convergence (we guarantee it will converge).
The full \ouralg~is given in Algorithm \ref{alg:mp_apmp}.


\subsection{Phase 1: Path Scheduling and Damping}
For convenience, we use the convention that chains go from ``left'' to ``right,'' where the left-most variable on
a chain corresponding to an augmenting path is $\var_{\augpath_1}$, and the right-most variable is $\var_{\augpath_N}$.
In these terms, a \emph{forward} pass is from left to right, and a \emph{backward} pass is from right to left.

Suppose that at the end of iteration $t-1$, the outgoing message from the unary factor at the start of the chain
used in iteration $t$, $\augpath = \augpath(t)$,
is $\msg{\potfun_{\augpath_1}}{\var_{\augpath_1}}^{(t-1)}(\var_{\augpath_1}) = (0, b)^T$.
If the factor increments its outgoing message in such a way as to guarantee that
$b + f \le \potval_\beta^{01}$ for all steps along $\augpath$, the messages as shown
in \figref{fig:msgs_flow} will be computed (see Corollary \ref{lem:sat_msgs} below).  
Later analysis will explain why this is desirable.
Accounting for message normalization, this
can be accomplished by limiting the change $\Delta m(.) = m^{(t)}(.) - m^{(t-1)}(.)$ in outgoing message from the first unary variable
on a path to be
$\Delta \msg{\potfun_{\augpath_1}}{\var_{\augpath_1}}(\var_{\augpath_1}) = (0, f)^T$.
We also constrain the increment in the backward direction to equal the increment in the forward direction.

\begin{figure}[t]
\centering
\subfigure[]{\label{fig:msgs_flow1}\includegraphics[width=.48\columnwidth]{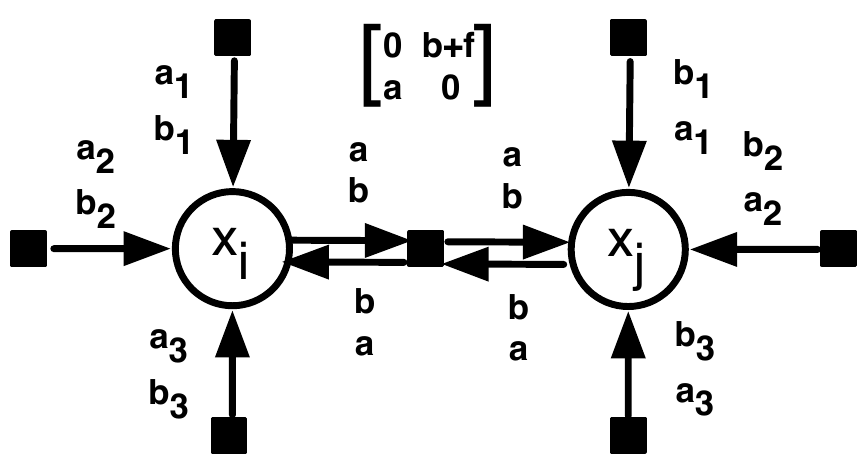}}
\hspace{1mm}
\subfigure[]{\label{fig:msgs_flow2}\includegraphics[width=.48\columnwidth]{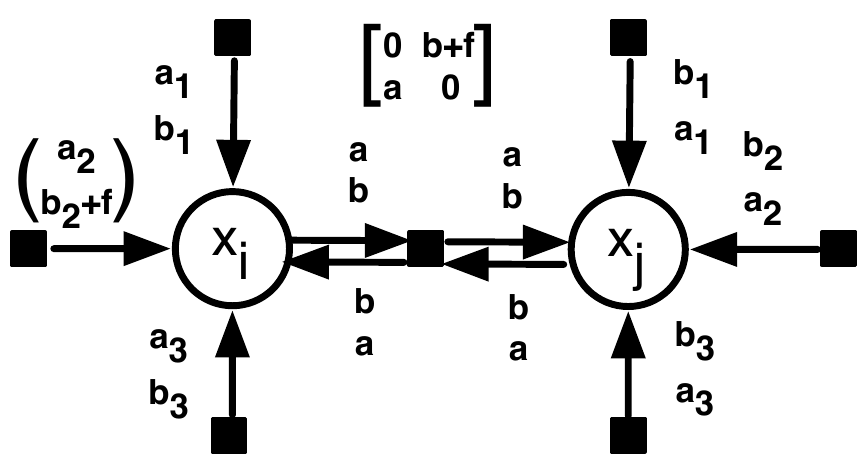}}
\subfigure[]{\label{fig:msgs_flow3}\includegraphics[width=.48\columnwidth]{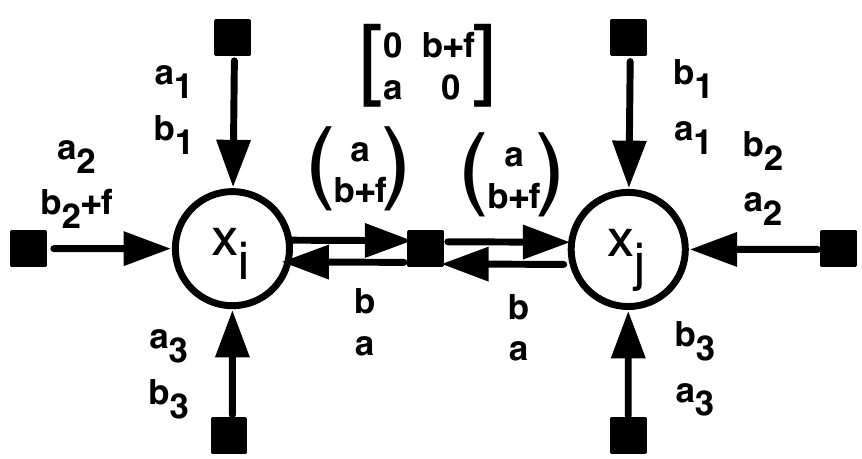}}
\hspace{1mm}
\subfigure[]{\label{fig:msgs_flow4}\includegraphics[width=.48\columnwidth]{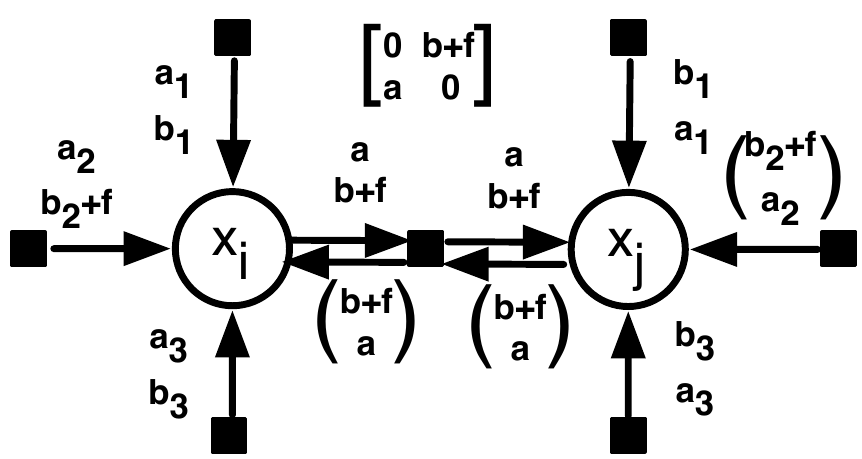}}
\caption{
The pairwise potential is in square brackets.  Only messages changed relative to
previous subfigure are shown in parentheses.  Let $a = a_1 + a_2 + a_3$ and $b = b_1 + b_2 + b_3$.  
\subref{fig:msgs_flow1} Start of iteration. The capacity of the edge $ij$ is $f$.
\subref{fig:msgs_flow2} Inductive assumption that each node on the augmenting
path will receive a message increment of $(0, f)$ from the left-neighbor.
\subref{fig:msgs_flow3} Passing messages completes the inductive step where
$\var_j$ receives an incremented message.
\subref{fig:msgs_flow4} Similarly, receiving an incremented message in the backwards
direction then updating messages from $j$ to $i$ completes the iteration.}
\label{fig:msgs_flow}
\end{figure}

\vspace{-.02in}
Under the constraints, the largest $f$ we can choose is \vspace{-.05in}
\begin{align} \label{eq:residual_capacity_in_msgs}
f  & =  \min \left(
	 \potval_{\augpath_N}^0 - \msg{\potfun_{\augpath_N}}{\var_{\augpath_{N}}}^{(t-1)}(0), \;\;
	\potval_{\augpath_1}^1 - \msg{\potfun_{\augpath_1}}{\var_{\augpath_{1}}}^{(t-1)}(1), \right. \nonumber \\
	& \left.  \min_{ij \mid \potfun_{ij} \in \potfun_\augpath} \left[ \potval_{ij}^{01} - \msg{\var_i}{\potfun_{ij}}^{(t-1)}(1)
	+ \msg{\var_i}{\potfun_{ij}}^{(t-1)}(0) \right] \right)
\end{align}
which is exactly the bottleneck capacity of the corresponding augmenting
path.  In other words, limiting the change in outgoing message value from
unary factors to be the bottleneck capacity of the augmenting path
will ensure that messages increments propagate through a chain unmodified--that
is, when one variable on the path receives an increment of $(0,f)^T$ as $\var_i$ does
in \figref{fig:msgs_flow2}, it will propagate the same increment to the next variable on
the path ($\var_j$), as in \figref{fig:msgs_flow3}.
This is proved in Lemma \ref{lem:sat_msgs}.

\begin{algorithm}[t]
\caption{Augmenting Paths Max-Product}
\label{alg:mp_apmp}
\begin{algorithmic}
\STATE $f(0) \gets \infty$ \\
\STATE $t \gets 0$ \\
\WHILE[Phase 1] {$f(t) > 0$}
	\STATE $\mathcal{T}(t), f(t) \gets \text{SCHEDULE}(\graphfg(t))$\\
	\STATE $\lambda_{\augpath_1}(t), \lambda_{\augpath_N}(t)  \gets \text{DAMPING}(\graphfg(t), \augpath(t), f(t))$ \\
	\STATE $ \graphfg(t+1) \gets \text{MP}(\graphfg(t), \edgesfg(\augpath(t)), \lambda_{\augpath_1}(t), \lambda_{\augpath_N}(t) )$\\
	\STATE $t \gets t + 1$
\ENDWHILE\\
\WHILE[Phase 2] {$\text{not converged}$}
\STATE \text{Run Strict MP} \\
\ENDWHILE\\
\end{algorithmic}
\end{algorithm}

{\bf Damping:}
The key, simple idea to the damping scheme is that we want unary factors to increment their
messages by the bottleneck capacity of the current chain.
The necessary value of $f$ can be achieved by damping the outgoing message from the
first and last unary potential on each chain.  For the first unary factor, if we previously have message
$(0, b)^T$ on the edge, then to produce message $(0, b + f)^T$, we can apply
damping $\lambda_{\augpath_1}(t)$ where $\lambda_{\augpath_1}(t)$ is chosen by solving the equation:
\begin{align}
\lambda_{\augpath_1}(t) \cdot b + (1 - \lambda_{\augpath_1}(t)) \cdot \potval_{\augpath_1}^1 & =  f + b \hbox{,}
\end{align}
yielding $\lambda_{\augpath_1}(t)   =  \frac{\potval_{\augpath_1}^1 - f -  b }{\potval_{\augpath_1}^1 - b }$.
The algorithm never chooses an augmenting path with 0 capacity, so we will never get
a zero denominator.

\commentout{
Interestingly, $f$ and $g$
will be increased at each iteration where
an augmenting path starts at $\potfun_{\augpath_1}$, and
$0 \le f \le f + g \le \potval_{\augpath_1}^1$, so $\lambda_{\augpath_1}(t)$ will
be non-decreasing at each iteration and
will approach 1 in the limit.  A message damping of 1 corresponds to leaving
the message on the edge unchanged, so this can then be seen as a scheme for
progressively damping messages, which does so more aggressively as inference progresses.
This is Algorithm \ref{alg:mp_apmp}.

With the prescribed damping, all forward messages increments will
propagate through the chain.
}
Analogous damping is applied in the opposite
direction.  This dynamic damping will then
produce the same message increments in the forward and backward
direction, which will be a key property used in later analysis.

{\bf SCHEDULE Implementation:} \label{sec:schedule}
The combination of potentials and messages on the edges contain the same information as the residual
capacities in the graph cuts residual graph.  Using this equivalence, any algorithm for finding
augmenting paths in the graph cut setting can be used to find chains to
pass messages on for MP.   The terms being minimized over in Eq.~\eqref{eq:residual_capacity_in_msgs} 
are residual capacities, which are defined in terms of messages and potentials.
Specifically, at the end of any iteration of the MP algorithm described in the next section,
 the residual capacities of edges between non-terminal nodes 
 can be constructed from potentials and current messages $m$ as follows:
\begin{align}
r_{ij} & = \potval_{ij}^{01} - \msg{\var_i}{\potfun_{ij}}(1)+ \msg{\var_i}{\potfun_{ij}}(0) \hbox{.}
\end{align}
The difference in messages $\msg{\var_i}{\potfun_{ij}}(1)- \msg{\var_i}{\potfun_{ij}}(0)$
is then equivalent to the difference in flows $f_{ij} - f_{ji}$ in the graph cuts formulation.
The residual capacities for terminal edges
can be constructed from messages and potentials related to unary factors:
\begin{align}
r_{si} & = \potval_i^1 - \msg{\potfun_i}{\var_i}(1) \\
r_{it} & =  \potval_i^0 - \msg{\potfun_i}{\var_i}(0)\hbox{.}
\end{align}

\subsection{Phase 2: Strict MP}
When the scheduler cannot find a positive-capacity path on which to pass messages,
it switches to its second phase and passes all messages at all iterations,
with no damping i.e., Strict MP.  It continues until reaching a fixed point. (We will
prove in \secref{sec:phase_2_analysis} that if potentials are finite, it will always converge).
The choice of Strict MP is not essential.  We can prove the same results
for any reasonable scheduling of messages.

\section{\ouralg~Phase 1 Analysis}
Assume that at the beginning of iteration $t$, each
variable $\var_i \in \varvec_{\augpath(t)}$ has received an incoming
message from its left-neighboring factor $\potfun_\alpha$,
$\msg{\potfun_\alpha}{\var_i}(\var_i)  =  (
					a,
					b)^T$.
We want to show that when each variable receives an incremented
message, $(a, b + f)^T$, the increment $(0, f)^T$---up to a normalizing constant---will be propagated through the
variable and the next factor, $\potfun_{ij}$, to the next variable on
the path.

The pairwise potential 
at the next pairwise factor  along the chain  will be
$\potfun_{ij}$. 
The damping scheme ensures that $\potval_{ij}^{10} \ge a$ and $\potval_{ij}^{01} \ge b + f$.
Lemma \ref{lem:sat_msgs} shows that under these conditions, factors will propagate messages unchanged.

\begin{lem}[Message-Preserving Factors]\label{lem:sat_msgs}
When passing standard MP messages with the factors as above, $\potval_{ij}^{10} \ge a$, and $\potval_{ij}^{01} \ge b + f$,
the outgoing factor-to-variable message is equal to the incoming
variable-to-factor message i.e.\ $\msg{\potfun_{ij}}{\var_j} = \msg{\var_i}{\potfun_{ij}}$
and $\msg{\potfun_{ij}}{\var_i} = \msg{\var_j}{\potfun_{ij}}$.
\end{lem}
\begin{proof} This follows from plugging in values to the message updates.  
See supplementary materials.
\end{proof}
Lemma \ref{lem:sat_msgs} allows us to easily compute value of all messages
passed during the execution of Phase 1 of APMP and thus the change in beliefs
at each variable.

\begin{cor}[Structured Belief Changes] \label{cor:msg_values}
Before and after an iteration $t$ of Phase 1 APMP, the change in unary belief at each variable
in $\varvec_{\augpath(t)}$ will be $(0, 0)^T$, up to a constant normalization.
\end{cor}

\begin{proof}
Under the APMP damping scheme, the change in message from the first unary factor
in $\augpath(t)$ will be $(0, f)^T$, and the change in message from the last unary
factor in $\augpath(t)$ will be $(f, 0)^T$ where $f$ is as defined in Eq.~\eqref{eq:residual_capacity_in_msgs}.
Without message normalization, these messages
will propagate unchanged through the pairwise factors in $\augpath(t)$ by Lemma \ref{lem:sat_msgs}.
Variable to factor messages will also propagate the change unaltered.

Message normalization subtracts a positive constant $c = \min(a, b+f)$ from both entries
in a message vector.  Existing message values will only get smaller, so the message-preserving
property of factors will be maintained.  Thus, each variable will receive a message change
of $(-c_L, f-c_L)^T$ from the left and a message change of $(f-c_R, -c_R)^T$ from the right.
The total change in belief is then $(f-c_L-c_R, f-c_L-c_R)^T$, which completes the proof.
\end{proof}

\figref{fig:msgs_flow} illustrates the structured message changes.

\subsection{Message Free View}

Here, using the reparametrization view of max-product from \shortciteA{Wainwright04},
we analyze the equivalent ``message-free'' version
of the first phase of \ouralg---one that directly modifies potentials rather than sending messages.
Corollary \ref{cor:msg_values} shows that all messages in APMP can be analytically computed.
We then use these message values to compute the change in parameterization due to the messages
at each iteration.  The main result in this section is that this change in parameterization
is exactly equivalent to that performed by graph cuts.

An important identity, which is a special case of the junction tree representation \shortcite{Wainwright04},
states that we can equivalently view MP on a tree as reparameterizing $\potfun$ according to beliefs $b$:
\begin{align}
&\sum_{i \in \nodes} \tilde \potfun_{i}(\var_i) + \sum_{ij \in \edges} \tilde \potfun_{ij}(\var_i, \var_j) \nonumber \\
 &\! = \sum_{i \in \nodes} b_{i}(\var_i)
	  + \sum_{ij \in \edges} \left[ b_{ij}(\var_i, \var_j) - b_{i}(\var_i) - b_{j}(\var_j) \right]  \label{eq:std_reparam}
\end{align}
where $\tilde \potfun$ is a reparametrization i.e.\ 
$E(\varvec; \potfun) = E(\varvec; \tilde \potfun)$ $\forall \varvec$.
\commentout{
This is a special case of Corollary 13.3 in \shortcite{Koller+Friedman:09}, which
states that MP in an arbitrary loopy cluster graph can be viewed
as reparameterizing the initial energy function --- or in other words, can be implemented
by using messages to directly change potentials.
}
At any point, we can stop and calculate current beliefs and apply the reparameterization
(i.e., replace original potentials with reparameterized potentials and set all messages
 to 0). This holds for
damped factor graph max-product even if factor to variable
messages are damped.  

{\bf ``Used'' and ``Remainder'' Energies: } 
To analyze reparameterizations, we begin by splitting $E$ into two components: a part
that has been used so far, and a remainder part.  The used part is defined as the energy
function that would have produced the current messages if no damping were used.  The
remainder is everything else.
Since damping is only applied at unary potentials, we assign all pairwise potentials
to the used component:  $\potfunu_{ij}(\var_i, \var_j) = \potfun_{ij}(\var_i, \var_j)$.
The used component of unary potentials can easily be defined as the current
message leaving the factor: $\potfunu_i(\var_i) = \msg{\potfun_i}{\var_i}(\var_i)$.
Consequently, the remainder pairwise potentials are zero, and the remainder
unary potentials are $\potfunr_i(\var_i) = \potfun_{i}(\var_i) - \potfunu_{i}(\var_i)$.
We apply the message-free interpretation to get a reparameterized version of
$E(\varvec; \potfunu)$ then add
in the remainder component of the energy unmodified.

{\bf Analyzing Beliefs:}
The parameterization in Eq.\ \eqref{eq:std_reparam} depends on unary and
pairwise beliefs.
We consider the change in beliefs from that defined by messages at the start of
an iteration of APMP to that defined by messages at the end of an iteration.
There are three cases to consider.

\textbf{Case 1}  Variables and potentials not in or neighboring $\varvec_{\augpath(t)}$
will not have any potentials or adjacent beliefs changed, so the
reparametrization will not change.

\textbf{Case 2} Potentials neighboring $\var \in \varvec_{\augpath(t)}$ but not in $\edgesfg(\augpath(t))$ could possibly
be affected by the belief at a variable in $\varvec_{\augpath(t)}$,
since the belief at an edge depends on the beliefs at
variables at each of its endpoints.  However, by Corollary \ref{cor:msg_values},
after applying standard normalization, this belief does
not change after a forward and backward pass of messages, so overall they are unaltered.

\textbf{Case 3}
We now consider the belief of potentials $\potfunu_{ij} \in \potfunu_{\augpath(t)}$.
This is the most involved case, where the parametrization
does change, but it does so in a very structured way.

\begin{lem}[] \label{lem:pairwise_reparam}
The change in pairwise belief on the current augmenting path $\augpath(t)$
from the beginning of an iteration $t$ to the end of an iteration is
\begin{align}
\Delta b_{ij}(\var_i, \var_j)& = \! \left[ \begin{array}{cc}
					\!\!0   & \!\!\! -f \!\! \\
					\!\! + f & \!\!\! 0 \!\!
				  \end{array} \right] + f\;\;\;\; \hbox{$ij \in \augpath(t)$.}
\end{align}
\end{lem}
\begin{proof}
This follows from applying the standard reparameterization \eqref{eq:std_reparam}
to messages before and after an iteration of Phase 1 \ouralg.  See supplementary material
for details.
\end{proof}

{\bf Unary Reparameterizations:}
  As discussed above, the used part of the energy
is grouped with messages and reparameterized as standard, while the remainder part is left
unchanged and is added in at the end:
\begin{align}
\tilde \potfun_i(\var_i) & = b_i(\var_i; \potfunu) + \potfunr_i(\var_i) \hbox{.}
\end{align}
Parameterizations defined in this way are proper reparameterizations of the original energy function.

\begin{lem} \label{lem:unary_reparam}
The changes in parameterization during iteration $t$ of Phase 1 APMP
at variables $\var_{\augpath_1}$ and $\var_{\augpath_{N}}$ respectively are
$(0, -f)^T$ and $(-f, 0)^T$.  The change in all other unary potentials is $(0,0)^T$.
\end{lem}
\begin{proof}
The Phase 1 damping scheme ensures that the message leaving the first factor on
$\augpath=\augpath(t)$ is incremented by $(0, f)^T$.  This means that $\potfunu_{\augpath_1}(\var_{\augpath_1})$
is incremented by $(0, f)^T$, so $\potfunr_{\augpath_1}(\var_{\augpath_1})$ is decremented by
$(0, f)^T$ to maintain the decomposition constraint.  Unary beliefs do not change, so
the new parameterization is then
$\Delta \potfun_{\augpath_1}(\var_{\augpath_1}) = \Delta b_{\var_{\augpath_1}}(\var_{\augpath_1}) +
\Delta \potfunr_{\var_{\augpath_1}}(\var_{\augpath_1}) = (0, -f)^T$.  A similar argument holds for
$\Delta \potfun_{\var_{\augpath_N}}$.

The only unary potentials involved in an iteration of APMP are endpoints of $\augpath(t)$, so no
other $\potfunr$ values will change.  The total change in parameterization at non-endpoint unary
potentials is then $(0, 0)^T$.
\end{proof}
{\bf Full Reparameterizations: }
Finally, we are ready to prove our first main result.
\begin{thm} \label{thm:final}
The difference between two reparametrizations induced by the messages in Phase 1 \ouralg, before and after
passing messages  on the chain corresponding to augmenting path $\augpath(t)$, is equal to the difference between reparametrizations of graph cuts before and after pushing flow through the equivalent augmenting path.
\end{thm}
\begin{proof}
The change in unary parameterization is given by Lemma \ref{lem:unary_reparam}.
The change in pairwise parameterization is
$\Delta \potfun_{ij}(\var_i, \var_j) = \Delta b_{ij}(\var_i, \var_j) - \Delta b_i(\var_i) - \Delta b_j(\var_j) = \Delta b_{ij}(\var_i, \var_j)$, where $\Delta b_{ij}(\var_i, \var_j)$ is given by Lemma \ref{lem:pairwise_reparam}.

Putting the two together, we see that the changes in potential entries
are exactly the same as those performed by graph cuts in
\eqref{eq:gc_reparametrization_start} - \eqref{eq:gc_reparametrization_end}:
\begin{align}
\Delta \potfun_{\augpath_1}(\var_{\augpath_1}) & =  \! \left[ \begin{array}{c}
					0 \\
					-f
				  \end{array} \right] \qquad\qquad \\
\Delta \potfun_{\augpath_N}(\var_{\augpath_N}) & =  \! \left[ \begin{array}{c}
					-f \\
					0
				  \end{array} \right] \qquad\qquad\\
\Delta \potfun_{\augpath_i}(\var_{\augpath_i}) & = \! \left[ \begin{array}{c}
					0\\
					0
				  \end{array} \right] \qquad \;\;\; \hbox{ $i \not = 1, N$}\\
\Delta \potfun_{{\augpath_i}, {\augpath_j}}(\var_{\augpath_i}, \var_{\augpath_j})& = \! \left[ \begin{array}{cc}
					\!\!0   & \!\!\! -f \!\! \\
					\!\! + f & \!\!\! 0 \!\!
				  \end{array} \right] \;\; \hbox{$\potfun_{{\augpath_i}, {\augpath_j}} \in \potfun_{\augpath(t)}$.}
\end{align}
This completes the proof of equivalence between Phase 1 \ouralg~and Phase 1 of graph cuts. 
\end{proof}

\figref{fig:kulesza} shows $\potfunu$ and Phase 1 \ouralg~messages  from running
two iterations on the example from \figref{fig:fixed_pts}.
\begin{figure}[tb]
\centering
\subfigure[]{\label{fig:kulesza2}\includegraphics[width=.49\columnwidth]{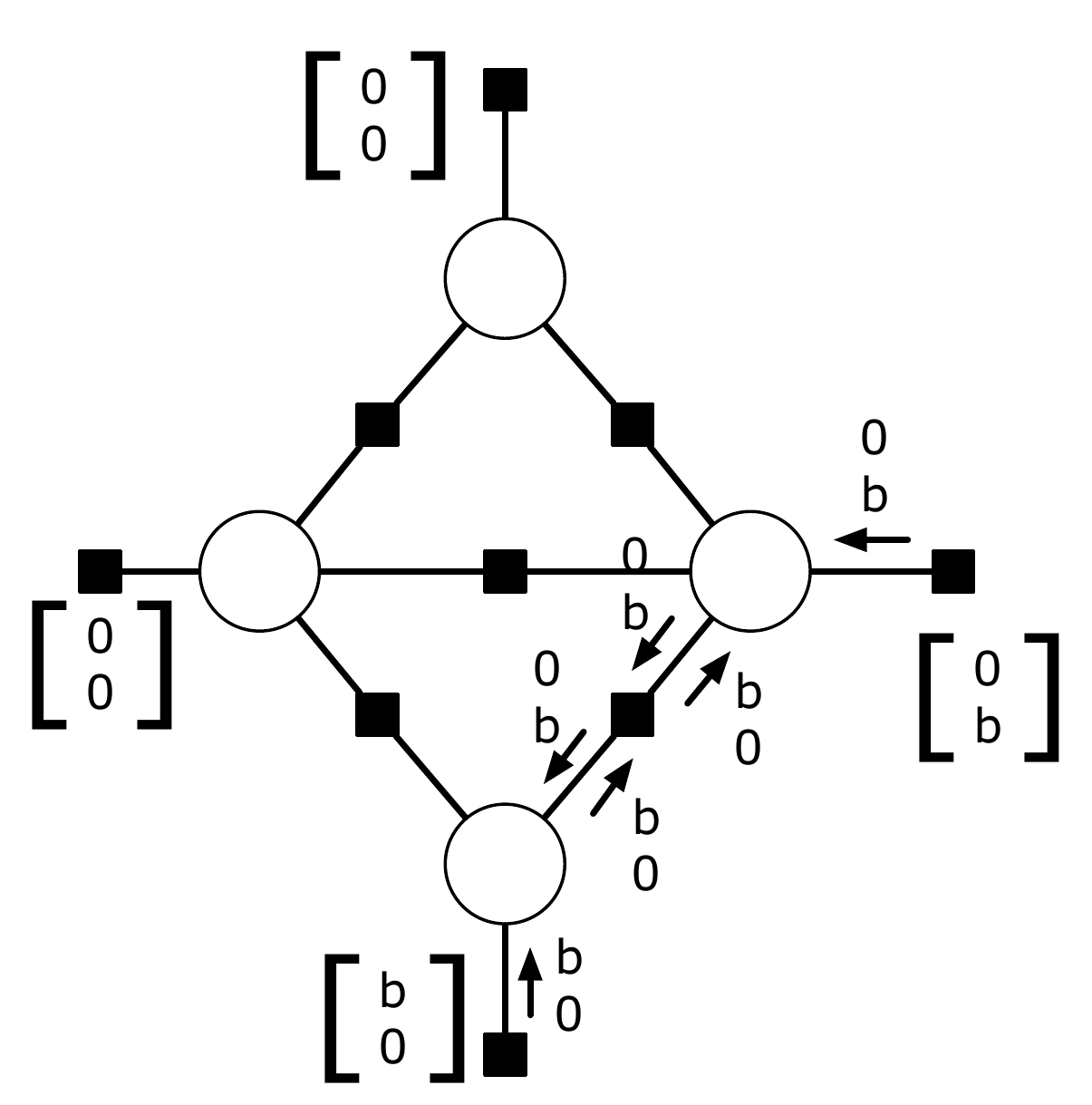}}
\hspace{2mm}
\subfigure[]{\label{fig:kulesza3}\includegraphics[width=.46\columnwidth]{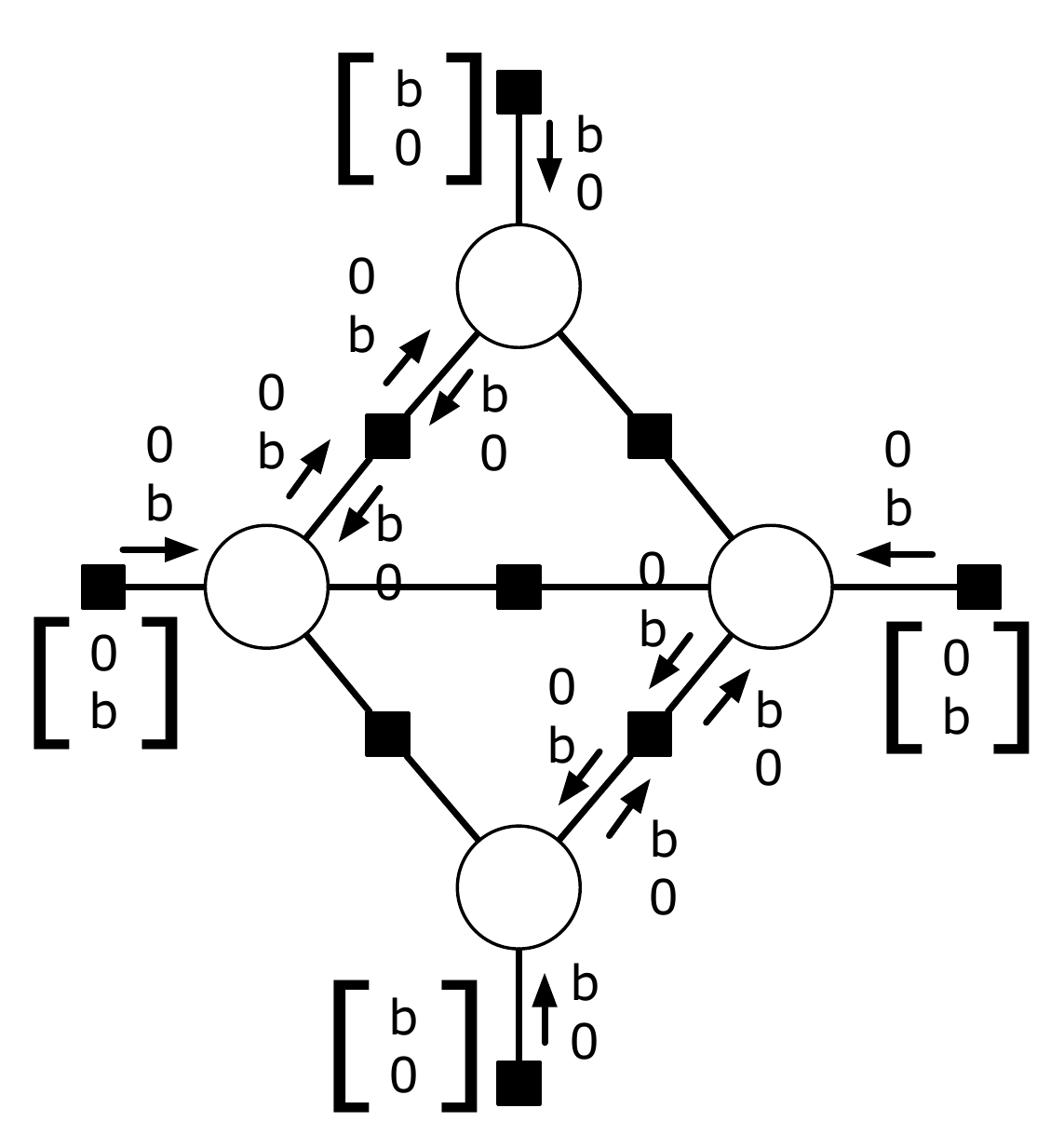}}
\caption{
Illustration of first two ``used'' energies and associated fixed points 
constructed by APMP on the problem from \figref{fig:fixed_pts}.
Potentials $\potfunu$ are given in square brackets.  Messages have no parentheses.  Edges
with messages equal to $(0, 0)$, and pairwise potentials, which are assumed strong, 
are not drawn to reduce clutter.
\subref{fig:kulesza2} First energy. 
\subref{fig:kulesza3}  Second energy. 
Note that both sets of messages give a max-product fixed point
for the respective energy.
}
\label{fig:kulesza}
\end{figure}

\section{\ouralg~Phase 2 Analysis} \label{sec:phase_2_analysis}
We now consider the second phase of \ouralg.  Throughout this section, we will
work with the reparameterized energy that results from applying the equivalent reparameterization view
of MP at the end of \ouralg~Phase 1---that is, we have applied the reparameterization to potentials, and
reset messages to 0.  All results could equivalently be shown by working with
original potentials and messages at the end of Phase 1, but the chosen presentation is simpler.

At this point, there are no
paths between a unary potential of the form $(0, a)^T, a>0$ and a unary
potential of the form $(b,0)^T, b>0$ with nonzero capacity.  
Practically, as in graph cuts, breadth first search could be used at this point to find an optimal
assignment.  However, we will show that running Strict MP leads to convergence to
an optimal fixed point.  This proves the existence of an optimal MP fixed point for any binary
submodular energy and gives a constructive algorithm (\ouralg) for finding it.

Our analysis relies upon the reparameterization at the end of Phase 1 defining what we
term \emph{homogeneous islands} of variables.
\begin{deff} A \emph{homogeneous island} is a set of variables $\varvec_{H}$ connected 
by positive capacity edges
 such that each variable $\var_i \in \varvec_H$ has normalized beliefs $(\alpha_i, \beta_i)^T$ where 
either $\forall i. \alpha_i=0$ or $\forall i. \beta_i=0$.  Further, after any number
of rounds of message passing amongst variables within the island, any message
$\msg{\potfun_{ij}}{\var_{j}}(\var_j)$
from a variable \emph{inside} the island $\var_i$ to a variable \emph{outside} the island $\var_j$ is identically 0,
and vice versa.
\end{deff}

Call the variables inside a homogeneous island with nonzero unary potentials \emph{seeds}
of the island.
\figref{fig:hom_islands} shows an illustration of homogeneous islands.
Homogeneous islands allow us to analyze messages independently within each island, 
without considering cross-island messages.

\begin{figure}[t]
\centering
\includegraphics[width=.5\columnwidth]{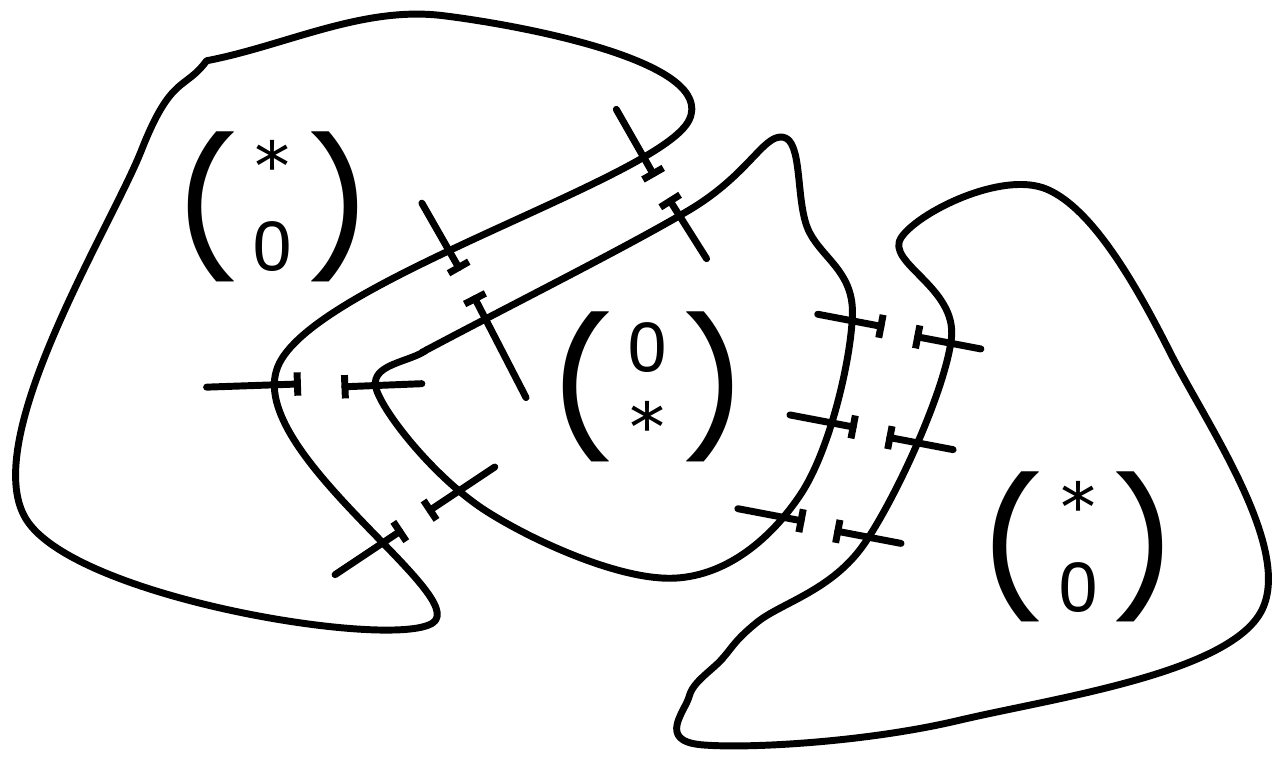}
\caption{
A decomposition into three homogeneous islands.  The left-most and right-most
islands have beliefs of the form $(\alpha,0)^T$, while the middle has beliefs
of the form $(0, \beta)^T$.  Non-touching cross-island lines indicate that messages passed from one
island to another will be identically 0 after any number of internal iterations of
message passing within an island.
}
\label{fig:hom_islands}
\end{figure}

\begin{lem} \label{lem:hom_islands}
At the end of Phase 1, the messages of \ouralg~define a collection of homogeneous islands.
\end{lem}
\begin{proof}
This is essentially equivalent to how the max-flow min-cut theorem proves that the
Ford-Fulkerson algorithm has found a minimum cut when no more augmenting paths
can be found.  The boundaries between islands are the locations of the cuts.
See supplementary material.
\end{proof}

Lemma \ref{lem:hom_islands} lets us analyze Strict MP independently within each
homogeneous island, because it shows that no non-zero messages will cross
island boundaries.  Thus, we can prove that internally, each island will reach a MP fixed point:
\begin{lem}[Internal Convergence] \label{lem:internal_convergence}
Iterating Strict MP inside a homogeneous island of the form $(\alpha, 0)^T$  (or $(0, \beta)^T)$) 
will lead to a fixed point where beliefs are of the form $(\alpha_i', 0)^T,\alpha_i' \ge 0$ (or $(0, \beta_i)^T, \beta_i'\ge 0$)
at each variable in the island.
\end{lem}
\begin{proof} (Sketch)
We prove the case where the unary potentials inside the island have form $(\alpha_i, 0)^T$.  The
case where they have form $(0, \beta_i)^T$ is entirely analogous.

At the beginning of Phase 2, all unary potentials will be of the form $(\alpha, 0)^T, \alpha \ge 0$.
By the positive-capacity edge connectivity of homogeneous islands property, 
messages of the form $(\alpha, 0)^T, \alpha > 0$ will eventually be propagated to all
variables in the island by Strict MP.
In addition, messages can only reinforce (and not cancel)
each other.  
For example, in a single loop homogeneous island, messages
will cycle around the loop, getting larger as unary potentials are added to incoming messages
and passed around the loop.  Messages will only stop growing when the 
the variable-to-factor messages become stronger than the pairwise potential.

On acyclic island structures, Strict MP will obviously converge.
On loopy graphs, messages will be monotonically increasing until they are
capped by the pairwise potentials (i.e., the pairwise potential is \emph{saturated}).  
The rate of message increase is lower bounded by
some constant (that depends on the strength of unary potentials and 
size of loops in the island graph, which are fixed), 
so the sequence will converge when all pairwise potentials are saturated.
\end{proof}

We can now prove our second main result:
\begin{thm}[Guaranteed Convergence and Optimality of \ouralg~Fixed Point]
\ouralg~converges to an optimal fixed point on binary submodular energy functions.
\end{thm}
\begin{proof}
After running Phase 2 of \ouralg, Lemma \ref{lem:internal_convergence}
shows that each homogeneous island will converge
to a fixed point where beliefs at all variables in the island can be decoded to 
give the same assignment as the initial seed of the island.
This is the same assignment as the optimal graph cuts-style
connected components decoding would yield.
Cross-island messages are all zero, and if a variable is not in an island, it has zero potential, sends and receives
all zero messages, and can be assigned arbitrarily.
Thus, we are globally at a MP fixed point, and beliefs can be decoded at
each variable to give the optimal assignment.
\end{proof}

\begin{figure}[t]
\hspace{-.18in}
\subfigure[Bad Fixed Point]{\label{fig:fp1}\includegraphics[width=.54\columnwidth]{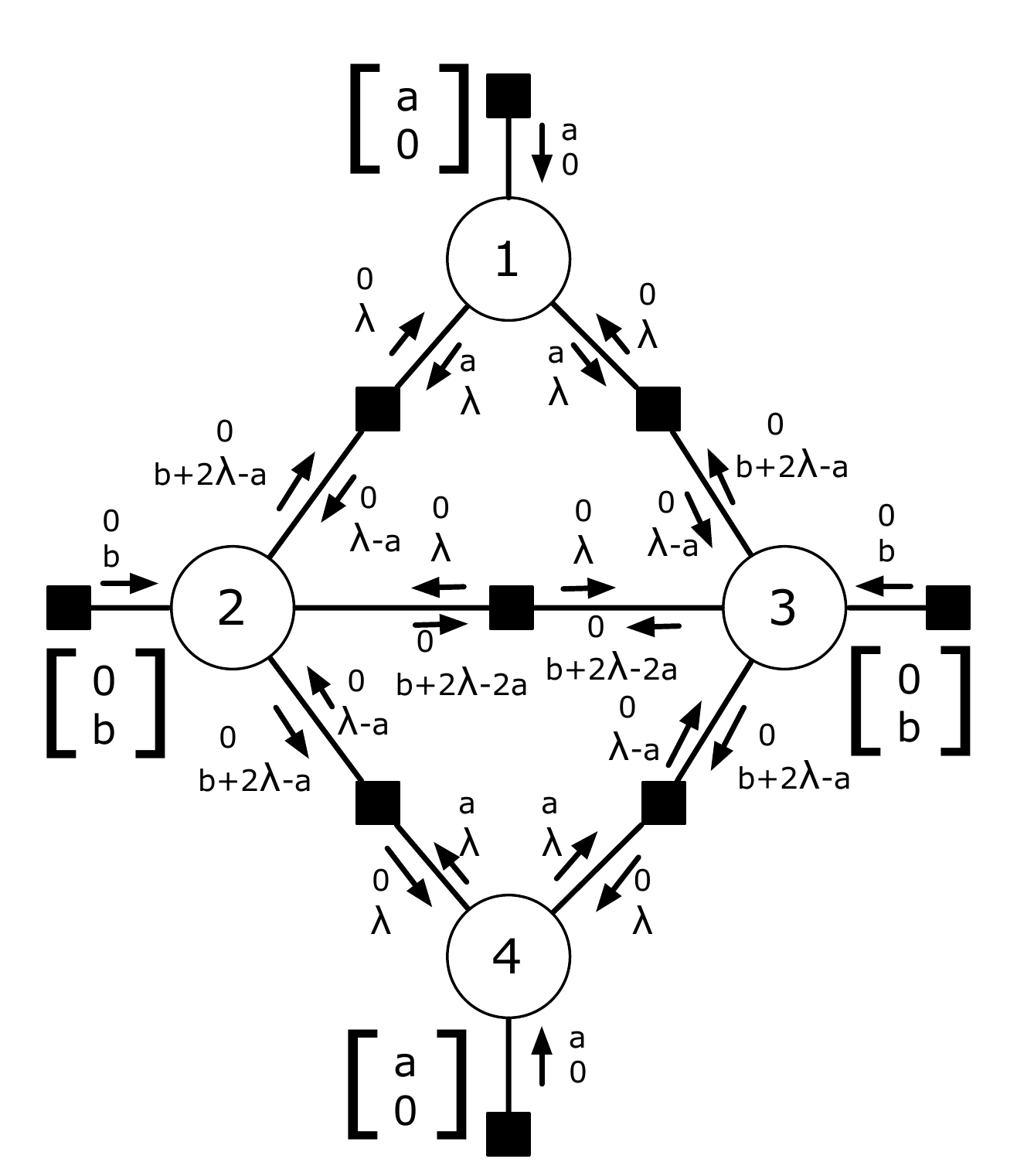}}
\hspace{-3mm}
\subfigure[Optimal Fixed Point]{\label{fig:fp2}\includegraphics[width=.52\columnwidth]{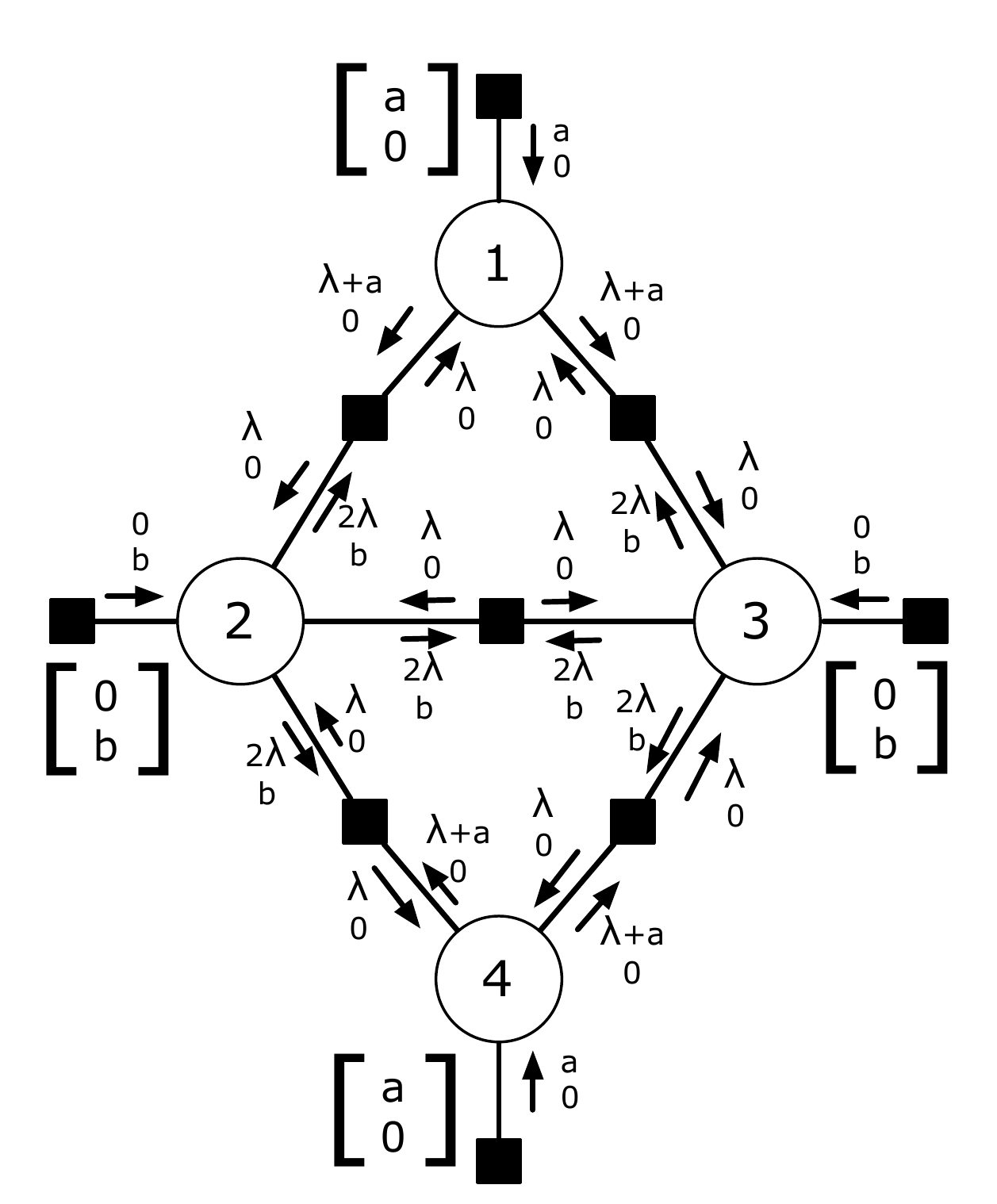}}
\caption{
The canonical counterexample used to show that MP is suboptimal on
binary submodular energy functions.
Potentials are given in square brackets.  Messages have no parentheses.  
Pairwise potentials are symmetric with strength $\lambda$, and 
$\lambda > 2 a > 2b$, making the optimal assignment $(1, 1, 1, 1)$.
\subref{fig:fp1} The previously analyzed fixed point.  Beliefs at 1 and 4 are  $(a, 2\lambda)^T$, and at 2 and 3 
are $(0,  b+3\lambda - 2a)^T$, which gives a suboptimal assignment.
\subref{fig:fp2} We introduce a second fixed point.
Beliefs at 1 and 4 are $(2\lambda+a, 0)^T$, and at 2 and 3 
are $(3\lambda, b)^T$, which gives the optimal assignment.
Our new scheduling and damping scheme 
 guarantees MP will find an optimal fixed point like this
  for any binary submodular 
energy function.
}
\label{fig:fixed_pts}
\end{figure}

Finally, we return to the canonical example used to show the suboptimality of MP 
on binary submodular energies.
The potentials and messages defining a suboptimal fixed point,
which is reached by certain suboptimal scheduling and damping schemes, 
are illustrated in \figref{fig:fixed_pts} \subref{fig:fp1}.
If, however, we run \ouralg, Phase 1 ends with the messages shown in \figref{fig:kulesza3} and Phase 2
converges to the fixed point shown in \figref{fig:fixed_pts} \subref{fig:fp2}.  Decoding beliefs from
the messages in \figref{fig:fixed_pts} \subref{fig:fp2} indeed gives the optimal assignment of $(1,1,1,1)$.

\section{Convergence Guarantees}

There are several variants of message passing algorithms for MAP inference
that have been theoretically analyzed.  There are generally two classes of results:
(a) guarantees about the optimality or partial optimality of solutions, assuming that
the algorithm has converged to a fixed point; and (b) guarantees about the
monotonicity of the updates with respect to some bound and whether the algorithm will converge.

Notable optimality guarantees exist for TRW algorithms \shortcite{Kolmogorov05} and MPLP
\shortcite{Globerson08}.  \shortciteA{Kolmogorov05} prove that fixed points of TRW satisfying
a \emph{weak tree agreement (WTA)} condition yield optimal solutions to binary submodular problems.
\shortciteA{Globerson08} show that if MPLP converges to beliefs with unique optimizing values,
then the solution is optimal.

Convergence guarantees for message passing algorithms are generally significantly weaker.
MPLP is a coordinate ascent algorithm so is guaranteed to converge; however, in general it
can get stuck at suboptimal points where no improvement is possible via updating the blocks
used by the algorithm.  Somewhat similarly, TRW-S is guaranteed not to decrease a lower bound.
In the limit where the temperature goes to 0, 
convexified \emph{sum-product} is guaranteed to converge to a solution of the standard
linear program relaxation, but 
this is not numerically practical to implement \shortcite{Weiss07}.
However, even for binary submodular energies, we are
unaware of results that guarantee convergence for convexified belief propagation, 
MPLP, or TRW-S in polynomial time.

Our analysis reveals schedules and message passing updates that guarantee convergence
 in low order polynomial time 
to a state where an optimal assignment can be decoded
for binary submodular problems.  This follows directly
from analysis of max-flow algorithms. 
By using shortest augmenting paths, the Edmonds-Karp algorithm converges
 in $O(|\nodes| |\edges|^2)$ time \shortcite{Edmonds72}.
 Analysis of the convergence time of Phase 2 is slightly more involved.
Given an island with a large single loop of $M$ variables, with strong pairwise potentials (say strength $\lambda$) 
 and only one small nonzero unary potential, say $(\alpha, 0)^T$, convergence
 will take on the order of $\frac{M \cdot \lambda}{\alpha}$ time, which could be large.
 In practice, though, we can reach the same fixed point by modifying nonzero unary potentials
 to be $(\lambda, 0)^T$, in which case convergence will take just order $M$ time.
 Interestingly, this modification causes Strict MP to become equivalent to the connected components
 algorithm used by graph cuts to decode solutions.
 

\section{Related Work}

There are close relationships between many MAP inference algorithms.  Here
we discuss the relationships between some of the more notable and similar algorithms.
\ouralg~is closely related to dual block coordinate ascent algorithms
discussed in \shortcite{Sontag09}---Phase 1 of \ouralg~can be seen as block coordinate
ascent in the same dual.  Interestingly, even though both are optimal ascent steps, \ouralg~reparameterizations 
are not identical to those of the sequential tree-block coordinate ascent algorithm in \shortcite{Sontag09}
when applied to the same chain.

Graph cuts is also highly related to the Augmenting DAG algorithm \shortcite{Werner06}.
Augmenting DAGs are more general constructs than augmenting paths, so with
a proper choice of schedule, the Augmenting DAG algorithm could also implement graph cuts.

Our work follows in the spirit of RBP \shortcite{Elidan06}, in that we
are considering dynamic schedules for belief propagation.  RBP is more general,
 but our analysis is much stronger.

Finally, our work is also related to the COMPOSE framework of \shortciteA{Duchi07}.
In COMPOSE, special purpose algorithms are used to compute MP messages for certain
combinatorial-structured subgraphs, including binary submodular ones.  We show here that
special purpose algorithms are not needed: the internal graph cut algorithm can be
implemented purely in terms of max-product.  Given a problem that contains a graph cut
subproblem but also has other high order or nonsubmodular potentials, our work shows how to
interleave solving the graph cuts problem and passing messages elsewhere in the
graph.

\section{Conclusions}
While the proof of equivalence to graph cuts was moderately involved, the
\ouralg~algorithm is a simple special case of MP.
The analysis technique is novel:
rather than
relying on the computation tree model for analysis, we directly
mapped the operations being performed by the algorithm to a known
combinatorial algorithm.
It would be interesting to consider whether there are other cases
where the MP execution might be mapped directly to a combinatorial
algorithm.

We have proven strong statements about MP fixed points on binary submodular
energies.  The analysis has a similar flavor to that of \shortciteA{Weiss00}, in that
we construct fixed points where optimal assignments can be decoded, but where
the magnitudes of the beliefs do not (generally) correspond to meaningful
quantities.  The strategy of isolating subgraphs might apply more broadly.  For
example, if we could isolate single loop structures as we isolate homogeneous
islands in Phase 1, a second phase might then be used to find optimal
solutions in non-homogeneous, loopy regions.

An alternate view of Phase 1 is that it is an intelligent initialization of messages
for Strict MP in Phase 2.  In this light, our results show that initialization can provably
determine whether MP is suboptimal or optimal, at least in the case of binary submodular 
energy functions.

The connection to graph cuts simplifies the space of MAP algorithms.
There are now precise mappings between ideas from graph cuts and ideas from belief propagation (e.g., augmenting
path strategies to scheduling).  It allows us, for example, to map the capacity scaling
method from graph cuts to schedules for message passing.

A broad, interesting direction of future work is to further investigate how insights related to 
graph cuts
can be used to improve inference in the more general
settings of multilabel, nonsubmodular, and high order energy functions.
At a high level, \ouralg~separates the concerns of improving the dual objective (Phase 1) from
concerns regarding decoding solutions (Phase 2).  In loopy MP, this delays overcounting of
messages until it is safe to do so.
We believe that this and other concepts presented here
will generalize.  We are currently exploring 
the non-binary, non-submodular case.

\subsubsection*{Acknowledgements}
We are indebted to Pushmeet Kohli for introducing us to the
reparameterization view of graph cuts and for many insightful
conversations throughout the course of the project.
We thank anonymous reviewers for valuable suggestions
that led to improvements in the paper.

\bibliographystyle{apacite}
\bibliography{dtarlow}

\newpage
\appendix

\section{Supplementary Material Accompanying ``Graph Cuts is a Max-Product Algorithm''}
We provide additional details omitted from the main paper due to space limitations.

{\bf Lemma \ref{lem:sat_msgs} } (Message-Preserving Factors)
\emph{When passing standard MP messages with the factors as above, $\potval_{ij}^{10} \ge a$, and $\potval_{ij}^{01} \ge b + f$,
the outgoing factor-to-variable message is equal to the incoming
variable-to-factor message i.e.\ $\msg{\potfun_{ij}}{\var_j} = \msg{\var_i}{\potfun_{ij}}$
and $\msg{\potfun_{ij}}{\var_i} = \msg{\var_j}{\potfun_{ij}}$.}

\begin{proof} This follows simply from plugging in values to the message updates.  We show the $i$ to $j$ direction.
\begin{eqnarray*}\vspace{-.1in}
\msg{\var_i}{\potfun_{{ij}}}(\var_i) & = & \left( \begin{array}{c}
					a \\
					b + f
				  \end{array} \right)\\
\msg{\potfun_{{ij}}}{\var_j}(\var_j) & = & \min_{\var_i} \left[ \potfun_{{ij}}(\var_i, \var_j) + \msg{\var_i}{\potfun_{{ij}}}(\var_i)  \right] \\
& = & \min_{\var_i} \left[ \potval_{ij}^{10} \cdot \var_i(1 - \var_j) + \potval_{ij}^{01} \cdot (1 - \var_i)\var_j \right. \\
&  & + \left. a \cdot (1 - \var_i) + (b + f) \cdot \var_i \right] \\
\msg{\potfun_{{ij}}}{\var_j}(0) & = & \min(\potval_{ij}^{10} + b + f, a) = a \\
\msg{\potfun_{{ij}}}{\var_j}(1) & = & \min(a + \potval_{ij}^{01}, b + f) = b + f \\
\msg{\potfun_{{ij}}}{\var_j}(\var_j) & = & \left( \begin{array}{c}
					a \\
					b + f
				  \end{array} \right)
\end{eqnarray*}
where the final evaluation of the $\min$ functions used the assumptions that $\potval_{ij}^{10} \ge a$ and $\potval_{ij}^{01} \ge b + f$.
\end{proof}

{\bf Lemma \ref{lem:pairwise_reparam} }
\emph{The change in pairwise belief on the current augmenting path $\augpath(t)$
from the beginning of an iteration $t$ to the end of an iteration is
\begin{align}
\Delta b_{ij}(\var_i, \var_j)& = \! \left[ \begin{array}{cc}
					\!\!0   & \!\!\! -f \!\! \\
					\!\! + f & \!\!\! 0 \!\!
				  \end{array} \right] + f\;\;\;\; \hbox{$ij \in \augpath(t)$.}
\end{align} }
\begin{proof}
At the start of the iteration, message $\msg{\var_i}{\potfun_{ij}}(\var_i) = (a, b)^T$ for some $a,b$.
As mentioned in the proof of Corollary \ref{cor:msg_values}, during APMP, $\msg{\var_j}{\potfun_{ij}}(\var_j)$
will be incremented by exactly the same values as $\msg{\var_i}{\potfun_{ij}}(\var_i)$, except in opposite
positions.  All messages are initialized to 0, so $\msg{\var_j}{\potfun_{ij}}(\var_j) = (b, a)^T$.
The initial belief is then
\begin{align}
b^{\hbox{\tiny{init}}}_{ij}(\var_i, \var_j)& = \! \left[ \begin{array}{cc}
					\!\!a + b   & \!\!\! \potval_{ij}^{01} + 2 a \!\! \\
					\!\! \potval_{ij}^{10} + 2 b & \!\!\! a + b \!\!
				  \end{array} \right] \\
& = \! \left[ \begin{array}{cc}
					\!\! 0   & \!\!\! \potval_{ij}^{01} + a - b \!\! \\
					\!\! \potval_{ij}^{10} + b - a & \!\!\! 0 \!\!
				  \end{array} \right] + \kappa_1 \hbox{.} 				
\end{align}
After passing messages on $\augpath(t)$, $\msg{\var_i}{\potfun_{ij}}(\var_i) = (a, b+f)^T$
and $\msg{\var_j}{\potfun_{ij}}(\var_j) = (b+f, a)^T$.  The new belief is
\begin{align}
b^{\hbox{\tiny{final}}}_{ij}(\var_i, \var_j)& = \! \left[ \begin{array}{cc}
					\!\!a + b + f   & \!\!\! \potval_{ij}^{01} + 2 a\!\! \\
					\!\! \potval_{ij}^{10} + 2 b +2 f & \!\!\! a + b + f\!\!
				  \end{array} \right] \\
& = \! \left[ \begin{array}{cc}
					\!\! 0  & \!\!\! \potval_{ij}^{01} + a - b - f\!\! \\
					\!\! \potval_{ij}^{10} + b - a + f & \!\!\! 0 \!\!
				  \end{array} \right] + \kappa_2 \hbox{.} 				
\end{align}
Here $\kappa_1 = a + b$ and $\kappa_2 = a + b + f$.
Subtracting the initial belief from the final belief finishes the proof:
\begin{align}
\Delta b_{ij}(\var_i, \var_j)& = \! \left[ \begin{array}{cc}
					\!\! 0  & \!\!\! -f\!\! \\
					\!\! f & \!\!\! 0\!\!
				  \end{array} \right] + f \hbox{.} 				
\end{align}
\end{proof}

{\bf Messages at the end of Phase 1 define homogeneous islands: }

We prove that messages
at the end of Phase 1 define homogeneous islands in two parts:

\begin{lem}[Binary Mask Property] \label{lem:binary_mask}
If a pairwise factor $\potfun_{ij}$ computes outgoing message $\msg{\potfun_{ij}}{j}(\var_j)=(0,0)^T$ given incoming
message $\msg{i}{\potfun_{ij}}(\var_i)=(\alpha,0)^T$ for some $\alpha > 0$, then it will compute the same
$(0,0)^T$ outgoing message given any incoming message of the form, $\msg{i}{\potfun_{ij}}(\var_i)=(\alpha',0)^T, \alpha' \ge 0$.
(The same is true of messages with a zero in the opposite position.)
\end{lem}
\begin{proof}
This essentially follows from plugging in values to message update equations.
Suppose $\msg{i}{\potfun_{ij}}(\var_i)=(\alpha,0)^T$
and $\msg{\potfun_{ij}}{j}(\var_j)=(0,0)^T$.  Plugging into
the message update equation, we see that,
\begin{eqnarray*}\vspace{-.1in}
\msg{\potfun_{{ij}}}{\var_j}(\var_j) & = & \min_{\var_i} \left[ \potfun_{{ij}}(\var_i, \var_j) + \msg{\var_i}{\potfun_{{ij}}}(\var_i)  \right] \\
& = & \min_{\var_i} \left[ \potval_{ij}^{10} \cdot \var_i(1 - \var_j) + \potval_{ij}^{01} \cdot (1 - \var_i)\var_j \right. \\
&  & + \left. \alpha \cdot (1 - \var_i) \right] \\
\msg{\potfun_{{ij}}}{\var_j}(0) & = & \min(\potval_{ij}^{10}, \alpha)  \\
\msg{\potfun_{{ij}}}{\var_j}(1) & = & \min(\alpha + \potval_{ij}^{01}, 0) = 0\\
\msg{\potfun_{{ij}}}{\var_j}(\var_j) & = & \left( \begin{array}{c}
					\min(\potval_{ij}^{10}, \alpha) \\
					0
				  \end{array} \right) \hbox{.}
\end{eqnarray*}
In order for this to evaluate to $(0,0)^T$ when $\alpha>0$, $\potval_{ij}^{01}$ must be 0.
Since $\potval_{ij}^{01} = 0$, no matter what value of $\alpha' \ge 0$ we are given, 
it is clear that $\min(\potval_{ij}^{10}, \alpha') = 0$.
\end{proof}

\begin{lem}[Iterated Homogeneity] \label{lem:iterated_homogeneity}
Homogeneous islands of type $(\alpha, 0)$ (or $(0, \beta)$) are closed under 
passing Strict MP messages between variables in the island.  That is, a variable that
starts with belief $(\alpha, 0)^T, \alpha \ge 0$ will have belief $(\alpha', 0)^T, \alpha' \ge 0$
after any number of rounds of message passing.
\end{lem}
\begin{proof}
Initially, all beliefs have the form $(\alpha_i, 0)^T, \alpha_i \ge 0$ by definition.
Given an incoming message of the
form $(\alpha, 0)^T, \alpha \ge 0$, a submodular pairwise factor will compute outgoing message
$(\min(\alpha, \potval_{ij}^{10}), 0)^T$, where $\potval_{ij}^{10} \ge 0$.  The minimum of
two non-negative quantities is positive.  Variable to factor messages will sum messages of
this same form, and the sum of two non-negative quantities is non-negative.  Thus, all messages passed within the island
will be of the form $(\alpha, 0)^T, \alpha \ge 0$, which beliefs will be of the proper form.
Lemma \ref{lem:binary_mask}
shows that edges previously defining the boundary of the island will still define the boundary
of the island.  
The case of incoming message $(0, \beta)^T$ is analogous.
\end{proof}

{\bf Lemma 4. }
\emph{At the end of Phase 1, the messages of \ouralg~define a collection of homogeneous islands.}
\begin{proof}
(Sketch) This is essentially equivalent to the max-flow min-cut theorem, which proves
the optimality of the Ford-Fulkerson algorithm when no more augmenting paths can
be found.  In our formulation, at the end of Phase 1, there are by definition no 
paths with nonzero capacity, which implies that along any path between a variable $i$
with belief $(\alpha,0)^T, \alpha>0$ and a variable $k$ with belief $(0,\beta)^T, \beta>0$,
there must be a factor-to-variable message that given incoming message $(\alpha, 0)^T, \alpha >0$
would produce outgoing message $(0,0)^T$. (This is similarly true of opposite direction messages.)

Thus, to define the islands, start at each variable will nonzero belief, say of the form $(\alpha, 0)^T$,
and search outwards
by traversing each edge iff it would pass a nonzero message given incoming message $(\alpha, 0)^T$.
Merge all variables encountered along the search into a single homogeneous island.
\end{proof}

\end{document}